%% file: main.tex
\documentclass{article}

\usepackage[preprint]{neurips_2024}

\usepackage[utf8]{inputenc} 
\usepackage[T1]{fontenc}    
\usepackage{hyperref}       
\usepackage{url}            
\usepackage{booktabs}       
\usepackage{nicefrac}       
\usepackage{microtype}      
\usepackage{xcolor}         

\usepackage{algorithm}
\usepackage{algorithmic}
\usepackage{amsmath, amsthm, amssymb, mathtools, amsfonts}
\usepackage{subcaption, array, booktabs}

\usepackage[capitalize,noabbrev]{cleveref}

\usepackage{thmtools, thm-restate}
\newtheorem{lemma}{Lemma}
\newtheorem{definition}{Definition}

\newcommand\norm[1]{\lVert#1\rVert}

\newcommand\set[1]{\left\{ #1 \right\}}

\newcommand\argmax{\arg\max}

\newcommand\Var{Var} 
\newcommand\E{\mathbb{E}} 
 
\newcommand \Tr{Tr}

\title{Motion Code: Robust Time Series Classification and Forecasting via Sparse Variational Multi-Stochastic Processes Learning}

\author{%
  Chandrajit Bajaj$^{*}$\\
  University of Texas at Austin
  \And
  Minh Nguyen$^{*}$\\
  University of Texas at Austin
}

\begin{document}

\maketitle

\begin{abstract}
Despite extensive research, time series classification and forecasting on noisy data remain highly challenging. The main difficulties lie in finding suitable mathematical concepts to describe time series and effectively separate noise from the true signals. Unlike traditional methods treating time series as static vectors or fixed sequences, we propose a novel framework that views each time series, regardless of length, as a realization of a continuous-time stochastic process. This mathematical approach captures dependencies across timestamps and detects hidden, time-varying signals within the noise. However, real-world data often involves multiple distinct dynamics, making it insufficient to model the entire process with a single stochastic model. To address this, we assign each dynamic a unique signature vector and introduce the concept of "most informative timestamps" to infer a sparse approximation of the individual dynamics from these vectors. The resulting model, called Motion Code, includes parameters that fully capture diverse underlying dynamics in an integrated manner, enabling simultaneous classification and forecasting of time series. Extensive experiments on noisy datasets, including real-world Parkinson's disease sensor tracking, demonstrate Motion Code’s strong performance against established benchmarks for time series classification and forecasting.
\end{abstract}

\input{tex/01Intro.tex}
\input{tex/02MotionCode}
\input{tex/03Experiments}

\input{tex/04Discussion}
\input{tex/05Conclusion}

\bibliographystyle{plainnat}
\bibliography{citation}

\newpage
\appendix
\input{tex/Appendix}


\end{document}

%% file: tex/01Intro.tex
\section{INTRODUCTION}\label{intro}
Noisy time series analysis is a challenging problem due to the difficulty in finding appropriate mathematical models to represent and study such data, unlike images or text. For example, consider two groups of time series, each representing the audio data of a word pronounced by speakers with different accents, and with varying lengths between 80 and 95 data points (see \cref{fig:uneven_lengths}). Common methods, like distance-based \citep{Jeong2011-sn} or shapelet-based \citep{Bostrom2017-kx} approaches, treat time series as ordered vectors, but fail when those vectors are highly mixed, with many red series resembling blue ones (see \cref{fig:uneven_length_combine}). Deep learning methods, such as recurrent neural networks (e.g., LSTM-FCN \citep{Karim2019-cj}) or convolutional networks (e.g., ROCKET \citep{Dempster2020-ck}), struggle to capture higher-order correlations in datasets with many data points but a limited number of individual series, as in this example. Techniques that rely on empirical statistics, such as dictionary-based \citep{Schafer2015-vw} or interval-based methods \citep{Deng2013-ap}, are also unreliable because noise can distort the collected statistics, making it difficult to separate noise from true signals. These challenges motivate our approach, which models time series collections using \textbf{stochastic processes} to recover underlying dynamics from noisy data.

\begin{figure}[htb]
  \centering
  \subfloat[2 Time Series Collections]{\includegraphics[scale=0.28]{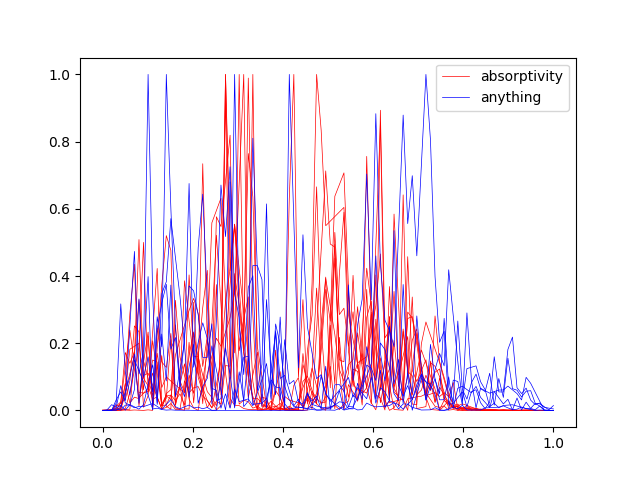}\label{fig:uneven_length_combine}}
  \hfil
  \subfloat[Absorptivity]{\includegraphics[scale=0.28]{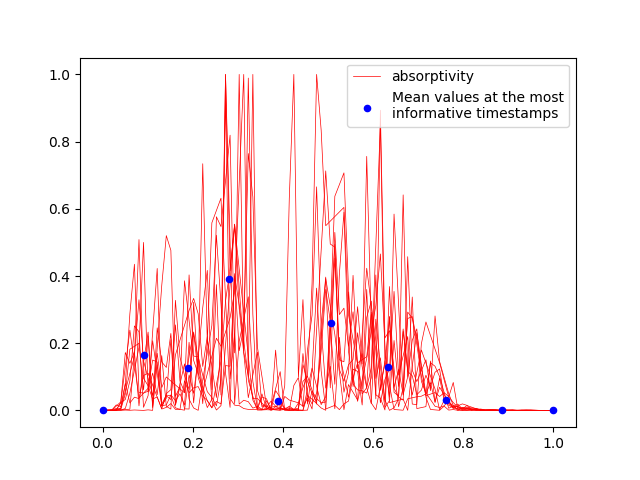}\label{fig:uneven_length_absorptivity}}
  \hfil
  \subfloat[Anything]{\includegraphics[scale=0.28]{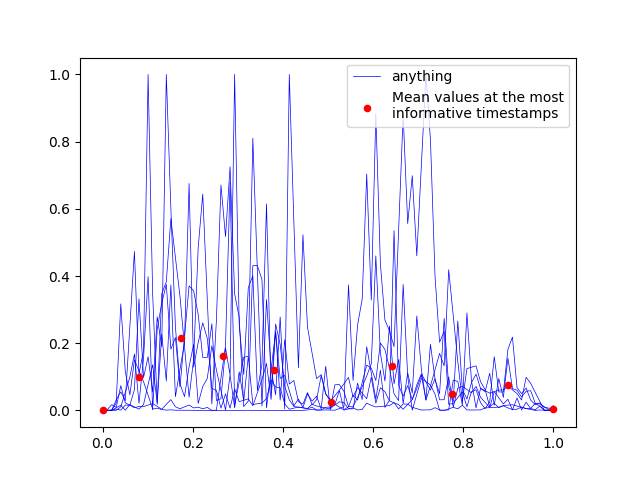}\label{fig:uneven_length_anything}}
  \caption{(a): Two Collections Of Time Series Representing Pronunciation Audio Data For The Words \textbf{Absorptivity} And \textbf{Anything}. (b) And (c): Most Informative Timestamps For The Pronunciation Of \textbf{Absorptivity} (Red) And \textbf{Anything} (Blue).} \label{fig:uneven_lengths}
\end{figure}

We propose modeling each time series as an instance of a common stochastic process governing the group's dynamics. While individual series may be noisy, our method recovers the underlying stochastic process, capturing key patterns such as increases, decreases, or stable phases, as shown by the skeleton approximations in \cref{fig:uneven_length_absorptivity} and \cref{fig:uneven_length_anything}. These approximations reveal core signals and the statistical relationships between data points.

Modeling time series with multiple dynamics, however, requires more than a single stochastic process. Unlike previous methods \citep{qi2010sparse,Durbin2012-jm,efficient_opt_sparse,pmlr-v206-moss23a} that focus on a single series, our framework introduces the most informative timestamps (see \cref{informative_timestamps_def}) to identify key features across multiple series. In addition, each dynamic model is assigned a signature vector, or motion code, which is jointly optimized using sparse learning techniques to prevent overfitting. This enables us to accurately distinguish between different dynamics. For instance, in \cref{fig:uneven_lengths}, scatter points reveal two distinct patterns that would otherwise be hidden in the noisy, mixed time series collections.

Combining these innovations, our proposed model, \textbf{Motion Code}, effectively learns from multiple underlying processes in a robust and comprehensive way. Our key contributions include:
\begin{enumerate}
    \item \textbf{Motion Code}: A model that jointly learns from noisy time series collections, explicitly modeling the underlying stochastic process to separate noise from core signals.
    \item \textbf{Irregular Data Handling}: Motion Code handles out-of-sync, varying-length, and missing data directly without interpolation, preserving temporal structure and avoiding alignment distortions.
    \item \textbf{Most Informative Timestamps}: An interpretable feature of Motion Code that employs variational inference to capture the essential dynamics in noisy time series.
\end{enumerate}

\subsection{Related Works}
\textbf{Time Series Classification}: Common techniques include distance-based methods \citep{Jeong2011-sn}, interval-based \citep{Deng2013-ap}, dictionary-based \citep{Schafer2015-vw}, shapelet-based \citep{Bostrom2017-kx}, feature-based \citep{Lubba2019-vg}, and ensemble models \citep{Lines2016-rv,Middlehurst2021-ja}. In deep learning, popular approaches involve convolutional neural networks \citep{Karim2019-cj,Dempster2020-ck}, residual networks \citep{Ma2016-vm}, autoencoders \citep{Hu2016-wo}.

\textbf{Time Series Forecasting}: Methods for forecasting include exponential smoothing \citep{Holt2004-oa}, TBATS \citep{De_Livera2011-av}, ARIMA \citep{Malki2021-ec}, probabilistic state-space models \citep{Durbin2012-jm}, and deep learning frameworks \citep{Lim2021-uh,Liu2021-zd}.

\textbf{Stochastic Modeling}: Gaussian processes \citep{Rasmussen2006-yi} are widely used for continuous-time series. To reduce computational cost, sparse Gaussian processes \citep{pmlr-v5-titsias09a} have been developed. Building on these approaches, advanced generative models with either approximate or exact inference \citep{qi2010sparse,efficient_opt_sparse,pmlr-v206-moss23a} have been introduced. However, these methods are typically limited to individual time series, whereas our approach extends to multi-time series collections, enabling joint modeling across different series.

The rest of the paper is organized as follows: Section 2 details the mathematical and algorithmic framework for Motion Code. Section 3 presents experiments and benchmarking for classification and forecasting tasks. Section 4 discusses the benefits of our framework.

%% file: tex/02MotionCode.tex
\section{MOTION CODE: JOINT LEARNING ON COLLECTIONS OF TIME SERIES}
\subsection{Stochastic Process Formulation and Data Assumption}\label{problem_statement}
We formulate the time series problem in the context of stochastic processes.

\textbf{Input}: The training data consists of samples from $L$ underlying stochastic processes $\set{G_k}_{k=1}^L$ where $L \in \mathbb{N} \geq 2$. For each $k \in \overline{1, L}$, let $\mathcal{C}_k$ represent the sample set of $B_k$ time series $\set{y^{i, k}}_{i=1}^{B_k}$ drawn from process $G_k$. Each time series $y^{i, k}$ has the timestamps $T_{i, k} \subset \mathbb{R}_{+}$, and at each timestamp $t \in T_{i, k}$, the corresponding data point is denoted as $y^{i, k}(t) \in \mathbb{R}$. The full dataset consists of $L$ collections of time series $\set{\mathcal{C}_k}_{k=1}^L$
    
\textbf{Tasks And Required Outputs:} The primary objective is to construct a model $\mathcal{M}$ that jointly learns the dynamics of the processes $\set{G_k}_{k=1}^L$ from the dataset $\set{\mathcal{C}_k}_{k=1}^L$. The model parameters must be transferable to the following tasks:
\begin{enumerate}
    \item \textbf{Classification}: Given a new time series $y = (y(t))_{t \in T}$ with timestamps $T$, classify it into one of the $L$ possible groups.
    \item \textbf{Forecasting}: For a time series $y$ generated by $G_k$ ($k \in \overline{1, L}$), predict future values at new timestamps $T$, i.e., predict $\set{y_t}_{t \in T}$.
\end{enumerate}

\textbf{Stochastic Process And Notations}: Recall that a stochastic process $G$ is defined as $\set{g(t)}_{t \geq 0}$, where $g$ is a random function, and $g(t)$ is the random data point at time $t$. The random function $g$ is referred to as the \textbf{underlying signal} of the process $G$. For any timestamps set $T$, let $g_T$ denote the signal vector $(g(t))_{t \in T} \in \mathbb{R}^{|T|}$.

\begin{figure}[htb]
  \centering
  \subfloat[Weekend]{\includegraphics[scale=0.35]{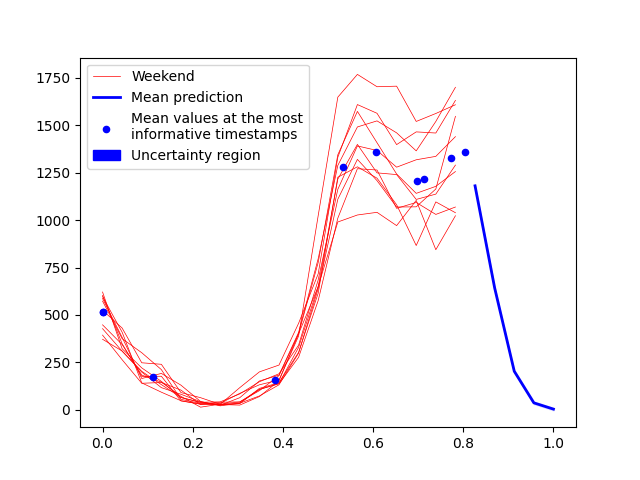}}
  \hfil
  \subfloat[Weekday]{\includegraphics[scale=0.35]{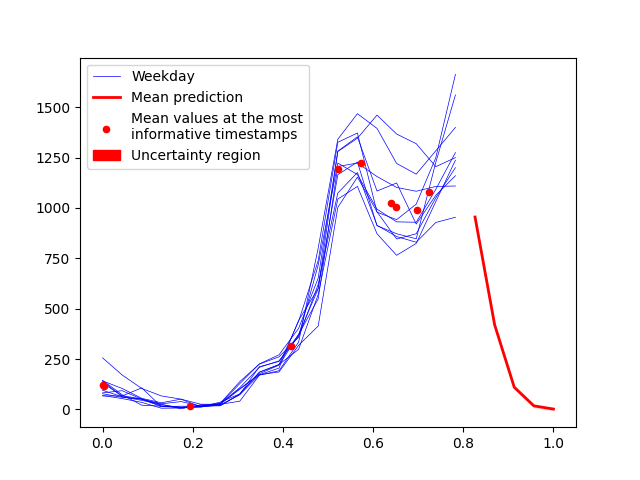}}
  \hfil
  \subfloat[Humidity Sensor]{\includegraphics[scale=0.35]{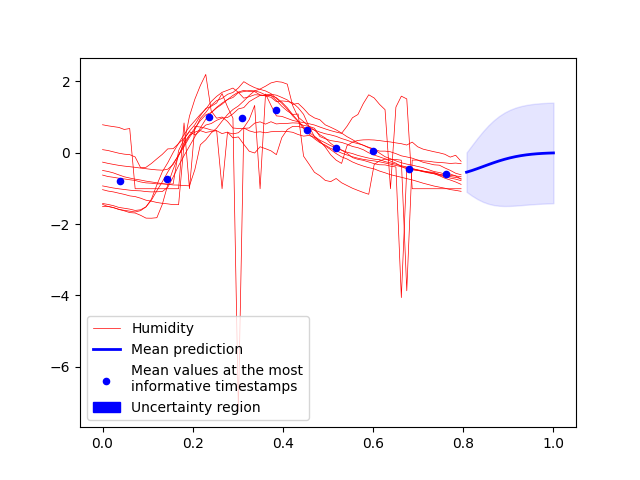}}
  \hfil
  \subfloat[Temperature Sensor]{\includegraphics[scale=0.35]{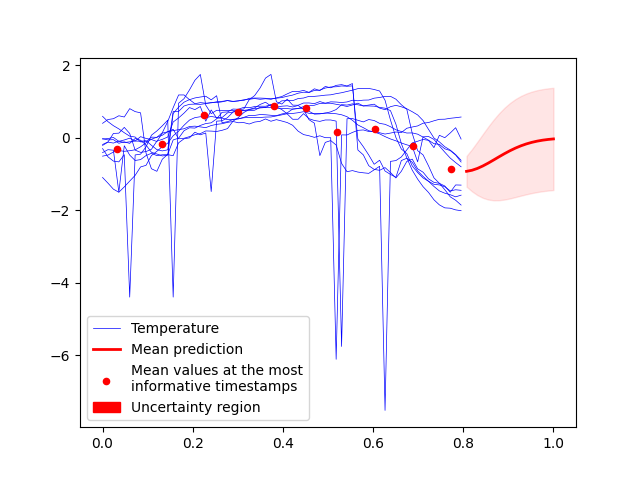}}
  \caption{Forecasting With Uncertainty For Time Series In \textbf{Chinatown} (Pedestrian Count On Weekends Vs Weekdays) And \textbf{MoteStrain} (Humidity Vs Temperature Sensor Values). \textbf{Motion Code} Is Trained On $[0, 0.8]$ And Predicted On $[0.8, 1]$.} \label{fig:most_informative_timestamp_prediction}
\end{figure}

\textbf{Data Assumptions:} We assume that the observed time series data points $y$ are normally distributed around the underlying signal of their respective stochastic processes. Specifically, let $g_k$ represent the underlying signal of the stochastic process $G_k$, i.e., $G_k = \set{g_k(t)}_{t \geq 0}$. Then, for $k \in \overline{1, L}$, the data $(y^{i, k})_{t \in T_{i, k}}$ assumes a Gaussian distribution with mean  $(g_k)_{T_{i, k}} = (g_k(t))_{t \in T_{i, k}} \in \mathbb{R}^{|T_{i, k}|}$ and covariance matrix $\sigma I_{|T_{i, k}|}$, where $I_n$ is the $n\times n$ identity matrix. The constant $\sigma \in \mathbb{R}_{+}$ is the unknown noise variance in the sample data from the underlying signals.

\subsection{The Most Informative Timestamps}\label{sec:most_informative_timestamps}
In this sub-section, we develop the core mathematical concept behind Motion Code called \textbf{the most informative timestamps}. The most informative timestamps of a time series collection generalize the concept of \textbf{inducing points} of a single time series introduced in \citep{pmlr-v5-titsias09a}. They are a small subset of timestamps that minimizes the mismatch between the original data and the information reconstructed using only this subset. The visualization of the most informative timestamps is provided in \cref{fig:most_informative_timestamp_prediction}, \cref{fig:sound_interpretable}, \cref{fig:parkinson_interpretable}, and is further discussed in \cref{discussion}.

To concretely define the most informative timestamps, we first introduce \textbf{generalized evidence lower bound function} (GELB) in \cref{ELBO_def}. We then define \textbf{the most informative timestamps} as the \textbf{maximizers} of this GELB function in \cref{informative_timestamps_def}.

\begin{definition}\label{ELBO_def}
Suppose we are given a stochastic process $G = \set{g(t)}_{t\geq 0}$ and a collection of time series $\mathcal{C} = \set{y^i}_{i=1}^B$ consisting of $B$ independent time series $y^i$ sampled from $G$. Each series $y^i = (y^i_t)_{t \in T_i}$ consists of $N_i = |T_i|$ data points and is called a \textbf{realization} of $G$. Let $m$ be a fixed positive integer. We define the \textbf{generalized evidence lower bound function} $\mathcal{L}=\mathcal{L}(\mathcal{C}, G, S^m, \phi)$ as a function of the data collection $\mathcal{C}$, the stochastic process $G$, the $m$-elements timestamps set $S^m = \set{s_1, \cdots, s_m} \subset \mathbb{R}_{+}$, and a variational distribution $\phi$ on $\mathbb{R}^m$ as follows:
\begin{equation}\label{eqn:ELBO_def}
\mathcal{L}(\mathcal{C}, G, S^m, \phi): = \frac{1}{B}\sum_{i=1}^B\int p(g_{T_i}|g_{S^m}) \phi(g_{S^m}) \log \frac{p(y^i|g_{T_i})p(g_{S^m})}{\phi(g_{S^m})}dg_{T_i} dg_{S^m}
\end{equation}
Recall that the vectors $g_{T_i}$ and $g_{S^m}$ are the signal vectors $(g(t))_{t \in T_i} \in \mathbb{R}^{|T_i|}$ and $(g(t))_{t \in S^m} \in \mathbb{R}^{|S^m|}$ on timestamps $T_i$ and $S^m$.
\end{definition}

\begin{definition}\label{informative_timestamps_def}
For a fixed $m \in \mathbb{N}$, the $m$-elements set $(S^m)^* \subset \mathbb{R}^+$ is said to be \textbf{the most informative timestamps} with respect to a noisy time series collection $\mathcal{C}$ of a stochastic process $G$ if there exists a variational distribution $\phi^*$ on $\mathbb{R}^m$ so that:
\begin{equation} \label{eqn:maximizer}
(S^m)^*, \phi^* = \argmax_{S^m, \phi} \mathcal{L}(\mathcal{C}, G, S^m, \phi)
\end{equation}
Also define the function $\mathcal{L}^{max}$ such that $\mathcal{L}^{max}(\mathcal{C}, G, S^m) 
:= \max_{\phi} \mathcal{L}(\mathcal{C}, G, S^m, \phi)$. Hence, $(S^m)^*$ can be found by maximizing $\mathcal{L}^{max}$ over all possible $S^m$.
\end{definition}

\subsection{Approximate Formula for $\mathcal{L}^{max}$}\label{sec:l_max_approx}
To compute the training loss function for Motion Code, we need to computationally approximate the function $\mathcal{L}^{max}$, which defines the \textbf{most informative timestamps} (see \cref{alg:motion_code}). Specifically, for a given set of $m$ timestamps $S^m$, a stochastic process $G$, and a collection $\mathcal{C}$ of $B$ time series $\set{y^i}_{i=1}^B$ sampled from $G$, our goal is to approximate $\mathcal{L}^{\max}(\mathcal{C}, G, S^m)$. This is achieved by approximating $G$ with a \textbf{kernelized Gaussian process} (see \cref{def:gaussian_process}), denoted as $H$, with a kernel function $K$.

\begin{definition} \label{def:gaussian_process}
A kernelized Gaussian process \citep{Rasmussen2006-yi} $H:= \set{h(t)}_{t \geq 0}$ with underlying signal $h$ is a stochastic process defined by the mean function $\mu:\mathbb{R} \to \mathbb{R}$ and the positive-definite kernel function $K:\mathbb{R} \times \mathbb{R} \to \mathbb{R}$. For the timestamps $T$, the joint distribution of the signal vector $h_T = (h(t))_{t \in T}$ is Gaussian and characterized by:
\begin{equation}
p(h_T)= p((h(t))_{t\in T}) = \mathcal{N}(\mu_T, K_{TT}), 
\end{equation}
Here $\mu_T$ is the mean vector $(\mu(t))_{t\in T}$, and $K_{TT}$ is the positive-definite $n \times n$ kernel matrix $(K(t, s))_{t, s \in T}$. $\mathcal{N}(\mu, \Sigma)$ denote a Gaussian distribution with mean $\mu$ and covariance matrix $\Sigma$.
\end{definition}

With the kernel $K$ of the approximate process $H$, for each $i \in \overline{1, B}$, define the kernel matrices $K_{T_iT_i}$, $K_{S^mT_i}$, and $K_{T_iS^m}$ as follows: $K_{T_iT_i} = (K(t, s))_{t \in T_i, s \in T_i}$, $K_{S^mT_i} = (K(t, s))_{t \in S^m, s \in T_i}$, and $K_{T_iS^m} = (K(t, s))_{t \in T_i, s \in S^m}$. From these, define the $|T_i|$-by-$|T_i|$ matrix $Q_{T_iT_i}: = K_{T_iS^m}(K_{S^mS^m})^{-1}K_{S^mT_i}$ for $ i \in \overline{1, B}$. Lastly, define the vector $Y$ and the joint matrix $Q^{\mathcal{C}, G}$ as follows:
\begin{equation}
Y = \begin{bmatrix} y^1 \\ \vdots \\ y^B \end{bmatrix},\ Q^{\mathcal{C}, G} = \begin{bmatrix}
        Q_{T_1T_1} & 0 & 0\\
        0 & \ddots & 0\\
        0 & 0 & Q_{T_BT_B}
\end{bmatrix}
\end{equation}

With these definitions, the function $\mathcal{L}^{max}$ can be approximated as:
\begin{equation}\label{eqn:ELBO_max_formula}
\mathcal{L}^{max}(\mathcal{C}, G, S^m) \approx \log p_\mathcal{N}(Y|0, B\sigma^2 I + Q^{\mathcal{C}, G}) - \frac{1}{2\sigma^2B}\sum_{i=1}^B \Tr(K_{T_iT_i}-Q_{T_iT_i})
\end{equation}
where $p_\mathcal{N}(X|\mu, \Sigma)$ denotes the density function of a Gaussian random variable $X$ with mean $\mu$ and covariance matrix $\Sigma$. The detailed proof for this approximation is given in the \textbf{Appendix}.

\subsection{Motion Code Learning}\label{motion_code_main}
With the core concept of \textbf{the most informative timestamps} outlined in \cref{sec:most_informative_timestamps} and approximation formula for $\mathcal{L}^{max}$ in \cref{sec:l_max_approx}, we can now describe \textbf{Motion Code} learning framework in details:

\textbf{Model And Parameters:} We approximate each stochastic process $G_k$ using a kernelized Gaussian process with a kernel function $K^{\eta_k}$, parameterized by $\eta_k$, for each $k \in \overline{1, L}$. All timestamps are normalized to the interval $[0, 1]$, and we select $m \in \mathbb{N}$ as the number of the most informative timestamps, as well as a fixed latent dimension $d \in \mathbb{N}$.

We jointly model the most informative timestamps $S^{m, k}$ for each stochastic process $G_k$ (with corresponding data collection $\mathcal{C}_k$) through a common mapping $\mathcal{G}: \mathbb{R}^d \to \mathbb{R}^m$. Specifically, we define $L$ distinct $d$-dimensional vectors $z_1, \dots, z_L \in \mathbb{R}^d$, referred to as \textbf{motion codes}, and use them to model $S^{m, k}$ as:
\begin{equation}
\widehat{S^{m, k}}:= \textbf{sigmoid}(\mathcal{G}(z_k)) \in \mathbb{R}^m
\end{equation}
where \textbf{sigmoid} is the standard sigmoid function. 

We approximate the map $\mathcal{G}$ with a linear transformation parameterized by a matrix $\Theta$, such that $\mathcal{G}(z_k) \approx \Theta z_k$. Thus, the Motion Code model involves three types of parameters:
\begin{enumerate}
    \item Kernel parameters $\eta:= (\eta_1, \cdots, \eta_L)$ to approximate underlying stochastic process $\set{G_k}_{k=1}^L$. 
    \item Motion codes $z:= (z_1, \cdots, z_L)$ with $z_i \in \mathbb{R}^d$.
    \item The joint map parameter $\Theta$ with dimension $m \times d$.
\end{enumerate}

\textbf{Training Loss Function:} The goal is to have $\widehat{S^{m, k}}$ closely approximate the true $S^{m, k}$, which maximizes $\mathcal{L}^{max}$. To achieve this, we maximize $\mathcal{L}^{max}(\mathcal{C}_k, G_k, \widehat{S^{m, k}})$ for all $k$, leading to the following loss function:
\begin{equation}\label{loss_fct}
\mathcal{U}(\eta, z, \Theta) = -\sum_{k = 1}^L \mathcal{L}^{max}(\mathcal{C}_k, G_k, \widehat{S^{m, k}}) + \lambda \sum_{k=1}^L \norm{z_k}_2^2  
\end{equation}
The first term is computed using the approximation formula for $\mathcal{L}^{max}$ in \cref{eqn:ELBO_max_formula}. The second term is a regularization term for the motion codes $z_k$, controlled by the hyperparameter $\lambda$. The full training procedure is detailed in \cref{alg:motion_code}.

\begin{algorithm}[htb]
\caption{Motion Code training algorithm}\label{alg:motion_code}
    \begin{flushleft}
        \textbf{Input:} $L$ collections of time series data $\mathcal{C}_k = \set{y^{i, k}}_{i=1}^{B_k}$, where the series $y^{i, k}$ has timestamps $T_{i, k}$, for $k \in \overline{1, L}$. Additional hyperparameters include number of the most informative timestamps $m$, motion codes dimension $d$, regularization parameter $\lambda$, max iteration $M$, and stopping threshold $\epsilon$.\\
        \textbf{Output:} Parameters $\eta, z, \Theta$ that optimize loss function $\mathcal{U}(\eta, z, \Theta)$ (see \cref{motion_code_main}).
    \end{flushleft}
    \begin{algorithmic}[1]
        \STATE Initialize $\eta$ and $z$ to be constant vectors $1$, and $\Theta$ to be the constant matrix, where each column is the arithmetic sequence between $0.1$ and $0.9$.
        \REPEAT
            \STATE Use the current parameter $\eta, z$ to calculate the predicted most informative timestamps for the $k^{th}$ stochastic process: $\widehat{S^{m, k}} = \textbf{sigmoid}(\Theta z_k)$. 
            \STATE Calculate $K^{\eta_k}_{S^{m, k}S^{m, k}}, K^{\eta_k}_{S^{m, k}T_{i, k}}, K^{\eta_k}_{T_{i, k}S^{m, k}}$ for $k \in \overline{1, L}, i \in \overline{1, B_k}$. Then calculate corresponding matrix $Q$'s, and $Q^{C, G}$ defined in \cref{sec:l_max_approx}.
            \STATE Use above calculations to compute $\mathcal{L}^{max}(\mathcal{C}_k, G_k, \widehat{S^{m, k}})$ approximated by \cref{eqn:ELBO_max_formula} via an automatic differentiation framework for each $k \in \overline{1, L}$.
            \STATE Calculate the loss $\mathcal{U}(\eta, z, \Theta)$  and its differential via automatic differentiation.
            \STATE Update parameters $(\eta, z, \Theta)$ using Limited-memory Broyden Fletcher Goldfarb Shanno (L-BFGS) algorithm \citep{L_BFGS_algo}. 
        \UNTIL{numbers of iterations exceed $M$ or training loss decreases less than $\epsilon$.}
        \STATE Output the final $(\eta, z, \Theta)$.
    \end{algorithmic}
\end{algorithm}

\subsection{Classification and Forecasting with Motion Code} \label{apply_to_timeseries}

We use the trained parameters $\eta, z, \Theta$ from \cref{alg:motion_code} to perform both time series forecasting and classification. The first step is to compute \textbf{preliminary predictions} that yield the predicted mean signal $p_k = \E[(g_k)_T] \in \mathbb{R}^{|T|}$, which forms the basis for these tasks.

\noindent\textbf{Preliminary Predictions:} For a given $k \in \overline{1, L}$, the predicted distribution of the signal vector $(g_k)_T$ is obtained by marginalizing over the signal $(g_k)_{S^{m, k}}$ at the most informative timestamps $S^{m, k}$ for process $G_k$:
\begin{equation}\label{predicted_dist_marginalization}
p((g_k)_T) = \int p((g_k)_T|(g_k)_{S^{m, k}}) \phi^*((g_k)_{S^{m, k}}) d(g_k)_{S^{m, k}}
\end{equation}
where the optimal variational distribution $\phi^*$ is defined in \cref{eqn:maximizer}. A detailed calculation of the distribution $p((g_k)_T)$ and its mean $p_k = \E[(g_k)_T]$, referred to as the \textbf{predicted mean signal}, is provided in the \textbf{Appendix}.

\textbf{Forecasting: } For the stochastic process $G_k$, the predicted mean signal $p_k = \E[(g_k)_T]$ serves as the forecast for the process.

\textbf{Classification: } To classify a series $y$ with timestamps $T$, we compute the predicted mean signal $p_k \in \mathbb{R}^{|T|}$ for each $k \in \overline{1, L}$. \textbf{Motion Code} outputs the predicted label based on the closest $p_k$, using the Euclidean distance $\norm{.}_{2, \mathbb{R}^{|T|}}$:
\begin{equation}
k_{\text{predicted}} = \argmax_k \norm{y - p_k}_{2, \mathbb{R}^{|T|}}
\end{equation}

\textbf{Time Complexity:} Matrix multiplication between an $m$-by-$m$ matrix and an $m$-by-$|T_{i, k}|$ or $|T_{i, k}|$-by-$m$ matrix is the most computationally expensive operation in \cref{alg:motion_code}. As a result, the time complexity of \cref{alg:motion_code} is $O\bigg(\sum_{k=1}^L\sum_{i=1}^{B_k}m^2|T_{i, k}| \bigg) \times M = O(m^2NM)$, where $N = \sum_{k=1}^L\sum_{i=1}^{B_k}|T_{i, k}|$ represents the total number of data points, $M$ is the maximum number of iterations, and $m$ is the number of most informative timestamps. For time series tasks, by the same argument, the cost of predicting a single mean vector $p_k$ is $O(m^2|T|)$. Thus, the cost for forecasting at timestamps $T$ is also $O(m^2|T|)$. For classification tasks across $L$ groups of time series, classifying a time series with timestamps $T$ has a complexity of $O(m^2|T||L|)$. Since $m$ is typically chosen to be small, these complexities are \textbf{approximately linear} in terms of the number of data points in the time series input.

%% file: tex/03Experiments.tex
\section{EXPERIMENTS}\label{experiment}
\subsection{Datasets}
We prepared three datasets for experimentation:

\textbf{Basic Sensor And Device Data}: Twelve publicly available time-series datasets were sourced from the UCR archive \citep{Bagnall2017-lb}, focusing on sensor and device data with corresponding ID from 1 to 12. Gaussian noise was added to simulate real-world conditions, with a standard deviation of 30\% of the maximum absolute value of the data points.

\textbf{Pronunciation Audio}: This dataset includes pronunciation audio from speakers with different accents (American, British, and Malaysian), focusing on two words: “absorptivity” and “anything.” These audio samples were sourced from publicly available recordings \citep{forvo}.

\begin{table*}[htb]
\centering
\caption{\label{classify_benchmarks_basic}Classification Accuracy (Percentage) For 7 Time Series Algorithms On Noisy Basic Sensor Datasets. The Highest Accuracy Is Highlighted In Red, The Second Highest In Blue.}
\begin{tabular}{|c|c|c|c|c|c|c|c|}
\hline
ID & DTW & TSF & RISE & BOSS & BOSS-E & catch22 & \parbox{30pt}{\centering Motion Code} \\[6pt]
\hline
1 & 54.23 & 61.22 & \color{blue}{65.6} & 47.81 & 41.69 & 55.39 & \color{red}{66.47} \\
\hline
2 & 54.47 & 58.07 & \color{blue}{59.35} & 50.06 & 58.42 & 52.85 & \color{red}{66.55} \\
\hline
3 & 52.42 & \color{blue}{54.28} & 53.79 & 50 & 50.95 & 53.58 & \color{red}{70.25} \\
\hline
4 & 92.7 & \color{red}{99.05} & \color{blue}{98.73} & 87.62 & 93.33 & 98.41 & 91.11 \\
\hline
5 & \color{blue}{57.98} & 57.14 & 42.02 & 52.1 & \color{blue}{57.98} & 45.38 & \color{red}{70.59} \\
\hline
6 & \color{red}{100} & \color{red}{100} & 83.13 & 99.2 & 91.97 & 95.98 & \color{red}{100} \\
\hline
7 & 57.05 & \color{blue}{68.71} & 65.79 & 52.77 & 53.26 & 55.88 & \color{red}{72.5} \\
\hline
8 & 21.92 & \color{blue}{28.77} & 26.03 & 12.33 & \color{blue}{28.77} & 24.66 & \color{red}{31.51} \\
\hline
9 & 56.47 & 61.1 & \color{blue}{61.5} & 53.83 & 53.51 & 57.19 & \color{red}{72.68} \\
\hline
10 & 78.33 & \color{blue}{92.22} & 85.56 & 65.56 & 77.22 & 80 & \color{red}{92.78} \\
\hline
11 & 63.27 & 67.68 & \color{blue}{69.78} & 48.06 & 61.91 & 64.43 & \color{red}{75.97} \\
\hline
12 & 78.25 & \color{red}{83.67} & 79.79 & 12.23 & 74.87 & 47.38 & \color{blue}{80.18} \\
\hline
\end{tabular}
\end{table*}

\begin{table*}[htb]
\centering
\caption{\label{classify_benchmarks_basic_2}Classification Accuracy (Percentage) For 7 Time Series Algorithms On Noisy Basic Sensor Datasets. ``Error" Indicates Failure To Run.}
\begin{tabular}{|c|c|c|c|c|c|c|c|} 
\hline
ID & \parbox{30pt}{\centering Shape-let} & Teaser & SVC & \parbox{30pt}{\centering LSTM-FCN} & Rocket & \parbox{30pt}{\centering Hive-Cote 2} & \parbox{30pt}{\centering Motion Code} \\[6pt]
\hline
1 & 61.22 & Error & 56.27 & \color{red}{66.47} & \color{blue}{62.97} & 61.52 & \color{red}{66.47} \\
\hline
2 & 52.61 & Error & 49.71 & 53.54 & \color{blue}{56.79} & 55.75 & \color{red}{66.55} \\
\hline
3 & 50 & 50.11 & 50 & 50 & 52.67 & \color{blue}{58.18} & \color{red}{70.25} \\
\hline
4 & 74.6 & 89.52 & 52.38 & 52.38 & 90.48 & \color{red}{98.73} & \color{blue}{91.11} \\
\hline
5 & 57.14 & 53.78 & 45.38 & 57.98 & 58.82 & \color{blue}{59.66} & \color{red}{70.59} \\
\hline
6 & 44.58 & \color{red}{100} & 85.94 & \color{red}{100} & 46.18 & \color{red}{100} & \color{red}{100} \\
\hline
7 & 62.49 & 63.17 & 49.85 & 61.61 & 70.36 & \color{red}{72.98} & \color{blue}{72.5} \\
\hline
8 & 20.55 & 21.92 & 26.03 & 17.81 & 27.4 & \color{red}{32.88} & \color{blue}{31.51} \\
\hline
9 & 47.76 & Error & 50.64 & 56.55 & \color{blue}{68.85} & 56.95 & \color{red}{72.68} \\
\hline
10 & 74.44 & 51.11 & 77.78 & 68.33 & 87.22 & \color{blue}{90} & \color{red}{92.78} \\
\hline
11 & 69.36 & 68.84 & 61.7 & 63.27 & 74.71 & \color{red}{78.49} & \color{blue}{75.97} \\
\hline
12 & 49.5 & 26.52 & Error & 12.67 & \color{red}{83.45} & 78.5 & \color{blue}{80.18} \\
\hline
\end{tabular}
\end{table*}

\textbf{Parkinson's Disease Sensor Data}: The Parkinson data are derived from the Clinician Input Study (CIS-PD) \citep{pd1,pd2}, a 6-month project using Apple Watch devices to monitor patients during clinic visits and at home. For two days before each clinic visit, patients reported symptoms every 30 minutes, focusing on medication state and tremor severity. The accelerometer data was segmented into 20-minute intervals (10 minutes before and after each symptom report). These Parkinson data were obtained from the Biomarker \& Endpoint Assessment to Track Parkinson’s disease DREAM Challenge. For up-to-date information on the study, visit \href{https://www.synapse.org/Synapse:syn20825169/wiki/600898}{https://www.synapse.org/Synapse:syn20825169/wiki/600898}.

We used two experimental settings for Parkinson's monitoring. The first tracks patients fully on medication state, distinguishing between no tremor and mild tremor to assess whether the patient has fully recovered or is still symptomatic. The second setting adds a third category for moderate to severe tremor, independent of medication state, aiming to capture broader tremor patterns, including cases where symptoms persist despite medication. This offers a more comprehensive assessment of tremor severity beyond recovery stages.

\subsection{Experimental Setups}
Motion Code was applied to three datasets: 12 basic datasets with added noise, pronunciation audio data, and Parkinson's sensor data, focusing on classification tasks. Forecasting was performed on the basic datasets and the audio dataset. All experiments were run on an Nvidia A100 GPU.

\textbf{Data Preprocessing}: For the Parkinson's dataset, we downsampled each segment by averaging per second, calculated the absolute differences between consecutive points, and applied an exponential moving average filter. We interpolated the data to 1,600 points for benchmark algorithms that require same-length time series, though Motion Code can handle misaligned data directly without the need for interpolation. More details are given in the \textbf{Appendix}.

\textbf{Kernel Choice}: We used a spectral kernel defined as $K^{\eta}(t, s):= \sum_{j=1}^J \alpha_j \exp(\scalebox{1.2}{-}0.5\beta_j|t-s|^2)$ with parameters $\eta = (\alpha_1, \cdots, \alpha_J, \beta_1, \cdots, \beta_J)$.

\textbf{Hyperparameters}: For experiments, we set $d = 2$, $\lambda = 1$, $\epsilon = 10^{-5}$, and $M = 10$. For basic and pronunciation audio datasets, we selected 10 most informative timestamps ($m=10$), and 1 kernel components ($J = 1$). For Parkinson’s disease (PD), we used $m= 6$, $J=2$ for the first setting, and $m=12$, $J=2$ for the second setting.

\subsection{Evaluation on Time Series Classification}
We compared Motion Code's performance on time series classification against 12 algorithms: \textbf{DTW} \citep{Jeong2011-sn}, \textbf{TSF} \citep{Deng2013-ap}, \textbf{RISE} \citep{Lines2016-rv}, \textbf{BOSS} \citep{Schafer2015-vw}, \textbf{BOSS-E} \citep{Schafer2015-vw}, \textbf{catch22} \citep{Lubba2019-vg}, \textbf{Shapelet} \citep{Bostrom2017-kx}, \textbf{Teaser} \citep{Schafer2020-ys}, \textbf{SVC} \citep{Loning2019-so}, \textbf{LSTM-FCN} \citep{Karim2019-cj}, \textbf{Rocket} \citep{Dempster2020-ck}, and \textbf{Hive-Cote 2} \citep{Middlehurst2021-ja}. We evaluated performance based on classification accuracy (measured in percentage).

As shown in \cref{classify_benchmarks_basic} and \cref{classify_benchmarks_basic_2}, Motion Code outperforms other algorithms on more than half of the noisy basic datasets and consistently ranks in the top 2, only behind the ensemble model Hive-Cote 2. This demonstrates the robustness of our method in handling collections of noisy time series.

\begin{table*}[htb]\small
\centering
\caption{\label{classify_benchmarks_realworld_1}Classification Accuracy For 7 Time Series Algorithms On Pronunciation Audio And Parkinson Data.}
\begin{tabular}{|c|c|c|c|c|c|c|c|} 
\hline
 \parbox{90pt}{\centering Data sets} & \parbox{25pt}{\centering Shape-let} & \parbox{25pt}{\centering Teaser} & \parbox{25pt}{\centering SVC} & \parbox{25pt}{\centering LSTM-FCN} & \parbox{25pt}{\centering Rocket} & \parbox{25pt}{\centering Hive-Cote 2} & \parbox{25pt}{\centering \textbf{Motion Code}} \\[6pt]
\hline
Pronunciation Audio & 68.75 & Error & 62.5 & 56.25 & {\color{blue}75} & {\color{blue}75} & {\color{red}87.5} \\
\hline
Parkinson setting 1 & 52.80 & 59.94 & {\color{blue}63.96} & 43.48 & 61.49 & 59.63 & {\color{red}70.81}\\
\hline
Parkinson setting 2 & 44.99 & 37.53 & 48.02 & 24.01 & {\color{blue}51.52} & 50.82 & {\color{red}54.31}\\
\hline
\end{tabular}
\end{table*}

\begin{table*}[htb]\small
\centering
\caption{\label{classify_benchmarks_realworld_2}
Classification Accuracy For 7 Time Series Algorithms On Pronunciation Audio And Parkinson Data.}
\begin{tabular}{|c|c|c|c|c|c|c|c|}
\hline
\parbox{90pt}{\centering Data sets} & \parbox{25pt}{\centering DTW} & \parbox{25pt}{\centering TSF} & \parbox{25pt}{\centering RISE} & \parbox{25pt}{\centering BOSS} & \parbox{25pt}{\centering BOSS-E} & \parbox{25pt}{\centering catch22} & \parbox{25pt}{\centering \textbf{Motion Code}} \\[6pt]
\hline
Pronunciation Audio & 50 & {\color{red}87.5} & 62.5 & {\color{blue}68.75} & 62.5 & 50 & {\color{red}87.5} \\
\hline
Parkinson setting 1 & 63.35 & 63.98 & {\color{red}70.81} & 61.80 & 65.53 & {\color{blue}68.94} & {\color{red}70.81}\\
\hline
Parkinson setting 2 & 43.12 & 51.98 & {\color{blue}53.61} & 45.92 & 36.83 & 51.52 & {\color{red}54.31}\\
\hline
\end{tabular}
\end{table*}

For real-world datasets, \cref{classify_benchmarks_realworld_1} and \cref{classify_benchmarks_realworld_2} show competitive performance from Motion Code compared to 12 other algorithms, highlighting its effectiveness in handling noise inherent in real-world data.

\subsection{Evaluation on Time Series Forecasting}
For forecasting, each dataset was split into two parts: 80\% of the data points were used for training, while the remaining 20\% of future data points were reserved for testing. For Motion Code, we generated a single prediction for all series within the same collection. We selected 5 algorithms as baselines for comparison: Exponential Smoothing \citep{Holt2004-oa}, ARIMA \citep{Malki2021-ec}, State Space Model \citep{Durbin2012-jm}, TBATS \citep{De_Livera2011-av}, and Last Seen, a basic method that uses previous values to predict the next time steps. 

\begin{table*}[htb]\small
\centering
\caption{\label{forecast_benchmarks} Average Root Mean-Square Error (RMSE) For 6 Time Series Forecasting Algorithms.}

\begin{tabular}{|c|c|c|c|c|c|c|} 
\hline
ID &  \parbox{45pt}{\centering Exp. Smoothing} & \parbox{45pt}{\centering ARIMA} &  \parbox{45pt}{\centering State space} & \parbox{45pt}{\centering Last seen} & \parbox{45pt}{\centering TBATS} & \parbox{45pt}{\centering \textbf{Motion Code}} \\ [6pt]
\hline
1 & Error & 1079 & 775.96 & 723.1 & {\color{blue}633.04} & {\color{red}518.49} \\
\hline
2 & 0.34 & 0.43 & 1.58 & {\color{blue}0.19} & {\color{red}0.17} & 0.27 \\
\hline
3 & 0.88 & 0.58 & 0.93 & {\color{blue}0.57} & {\color{red}0.56} & 0.74 \\
\hline
4 & 60.38 & 128.44 & 59.83 & {\color{blue}41.51} & {\color{red}20.94} & 417.94 \\
\hline
5 & 1117 & 3386 & 730.51 & {\color{red}497.3} & {\color{blue}560.88} & 648.27 \\
\hline
6 & 0.043 & 0.095 & 0.25 & {\color{red}0.019} & {\color{blue}0.02} & 0.048 \\
\hline
7 & Error & 2.02 & 2.37 & 1.24 & {\color{blue}0.96} & {\color{red}0.67} \\
\hline
8 & 1.7 & 2.85 & 1.7 & {\color{blue}1.35} & {\color{blue}1.35} & {\color{red}1.08} \\
\hline
9 & 1.11 & 1.52 & 1.09 & 1.01 & {\color{blue}0.88} & {\color{red}0.82} \\
\hline
10 & 3.38 & 4.85 & 4.41 & 1.77 & {\color{blue}1.72} & {\color{red}1.15} \\
\hline
11 & 2.79 & 2.01 & 3.21 & {\color{red}1.39} & {\color{blue}1.52} & 2.26 \\
\hline
12 & 4.37 & 5 & 4.45 & {\color{red}0.98} & {\color{blue}1.44} & {\color{red}0.98} \\
\hline
Audio & 0.087 & 0.27 & 0.086 & 0.1 & {\color{red}0.059} & {\color{blue}0.085} \\
\hline
\end{tabular}
\end{table*}

We ran the 5 baseline algorithms with individual predictions and compared the results, as shown in \cref{forecast_benchmarks}. Despite not making individual predictions for each series, Motion Code outperformed other methods in the majority of datasets. 

\textbf{Code: } The implementation is available at \href{https://github.com/mpnguyen2/motion\_code}{https://github.com/mpnguyen2/motion\_code}.

%% file: tex/04Discussion.tex
\section{MOTION CODE'S BENEFITS}\label{discussion}
\subsection{Interpretable Features}

Despite having several noisy time series that deviate from the common mean, the points at most informative timestamps $S^{m, k}$ form a skeleton approximation of the underlying stochastic process. All the important twists and turns are constantly observed by the corresponding points at important timestamps (see \cref{fig:most_informative_timestamp_prediction}). Those points create a feature that helps visualize the underlying dynamics with explicit global behaviors such as increasing, decreasing, staying still, unlike the original complex time series collections with no visible common patterns among series.

\begin{figure}[htb]
  \centering
  \subfloat[Absorptivity]{\includegraphics[scale=0.28]{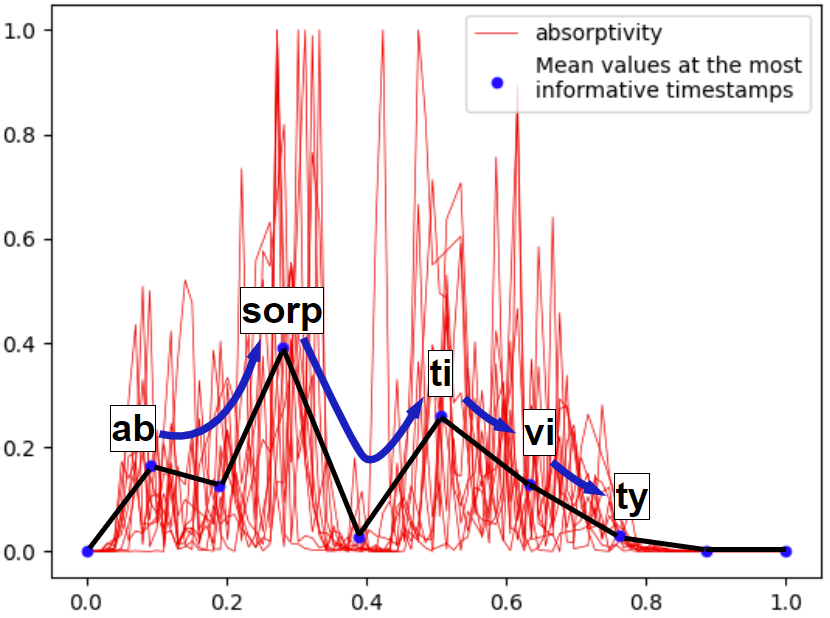}\label{fig:absorptivity}}
  \hfil
  \subfloat[Anything]{\includegraphics[scale=0.28]{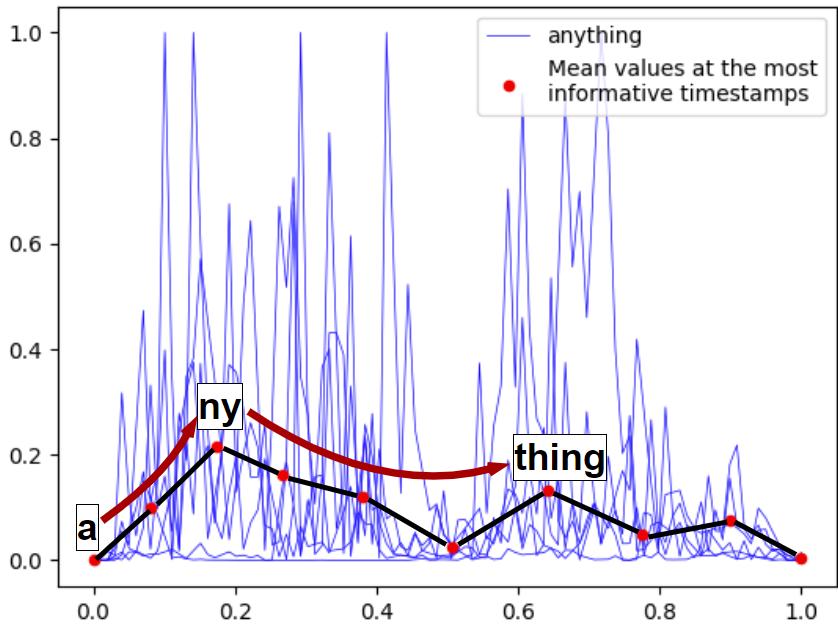}\label{fig:anything}}
  \caption{(a) And (b): Most Informative Timestamps For The Pronunciation Of \textbf{Absorptivity} And \textbf{Anything}, Highlighting Key Linguistic Components.} \label{fig:sound_interpretable}
\end{figure}

\textbf{Pronunciation Audio}: For pronunciation audio data, where speakers from different nationalities pronounce complex words, Motion Code highlights key linguistic features. For ``absorptivity" (ab-sorp-ti-vi-ty), the most informative timestamps align with significant phonetic components, identifying emphasis on ``ab" and ``sorp" followed by a notable silent pause before proceeding to ``ti" and then ``vi-ty" (see \cref{fig:absorptivity}). Similarly, for ``anything", it captures a strong vocal raise on ``a-ny" and emphasis on ``thing", preserving core pronunciation patterns across accents despite variations (see \cref{fig:anything}). This ability to focus on key moments reveals common speech dynamics shared across accents.

\textbf{Parkinson's Disease Data}: When tracking normal movement, the series appears random with no clear pattern. This reflects the unpredictable nature of normal motion, where no consistent behavior or tremor can be observed (see \cref{fig:normal}). In contrast, for patients with light tremor, the extracted timestamps reveal a more consistent, repetitive pattern, characterized by slight oscillations that correspond to minor, controlled hand swings. These small fluctuations, captured at key timestamps, represent typical behavior in light tremor (see \cref{fig:light_tremor}). For more severe tremors, the timestamps highlight a progression from smaller, repeated movements to larger, more exaggerated swings. Initially, the differences between consecutive data points are minimal, but as the tremor worsens, the fluctuations become more pronounced, with larger variations visible at critical timestamps (see \cref{fig:noticeable_tremor}). This interpretable feature allows us to track the severity and progression of tremors over time, offering valuable insights into patient's conditions.

\begin{figure}[htb]
  \centering
  \subfloat[Normal]{\includegraphics[scale=0.3]{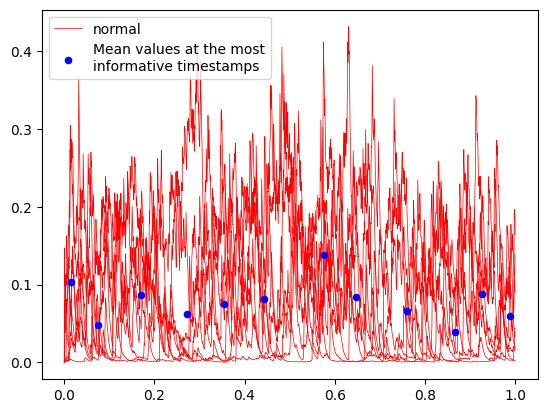}\label{fig:normal}}
  \hfil
  \subfloat[Light Tremor]{\includegraphics[scale=0.3]{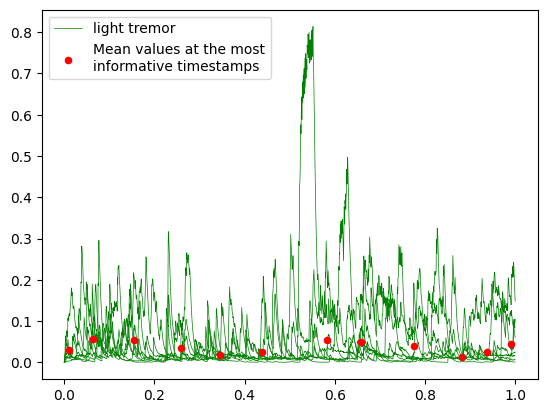}\label{fig:light_tremor}}
  \hfil
  \subfloat[Noticeable Tremor]{\includegraphics[scale=0.3]{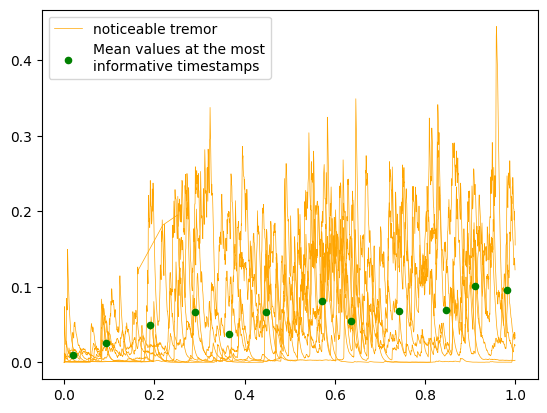}\label{fig:noticeable_tremor}}
  \caption{Interpretable Features Showing Tremor Patterns And Disease Stages For Parkinson Data: (a) Normal, (b) Light Tremor, And (c) Noticeable Tremor.}
  \label{fig:parkinson_interpretable}
\end{figure}

\subsection{Uneven Length and Missing Data}
Motion Code processes each data point and its timestamp independently, allowing the algorithm to handle time series with different timestamps. Despite the time series having uneven lengths and out-of-sync timestamps, Motion Code maintains accurate skeleton approximations (see \cref{fig:uneven_lengths}), demonstrating its effectiveness with incomplete and varying-length data.

For Parkinson's dataset, Motion Code efficiently handles out-of-sync timestamps and missing values. The time series from wearable sensors vary in length from 200 to 1660 points, with intermediate lengths such as 500 and 1000 points. Traditional methods often require interpolation to standardize these lengths, which can distort the data, especially when dealing with large disparities. Motion Code bypasses this need, processing time series of different lengths directly and learning across classes without interpolation, preserving the original data's integrity.

This capability is particularly useful for monitoring Parkinson's disease, where tremors and bradykinesia fluctuate, and sensor readings are irregular due to patient activities. Unlike other methods that struggle with asynchronous data, Motion Code treats each reading as part of an underlying stochastic process, enabling it to handle noisy, incomplete, and unsynchronized data efficiently. This eliminates the need for strict time-aligned monitoring, allowing patients to maintain natural schedules while ensuring accurate symptom tracking. Clinicians also benefit from clear, actionable insights, improving their ability to monitor disease progression and make timely interventions.

%% file: tex/05Conclusion.tex
\section{CONCLUSION}
In this work, we developed an integrated framework called \textbf{Motion Code}, utilizing variational inference and sparse stochastic process modeling. Unlike most existing methods focusing on either classification or forecasting, Motion Code performs both tasks simultaneously across diverse time series collections. Our model demonstrates robustness to noise and consistently achieves competitive performance against other leading time series algorithms. As discussed in \cref{discussion}, \textbf{Motion Code} offers interpretable features that capture the core dynamics of the underlying stochastic process. This is especially useful in domains like Parkinson’s disease monitoring, where understanding key patterns offers actionable insights for clinicians. Additionally, it handles varying-length time series and missing data, challenges that many other methods struggle with. In future work, we aim to extend \textbf{Motion Code} by incorporating non-Gaussian approximation to adapt to time series from different application domains.

\begin{ack}
\textbf{Funding information:} This research was supported in part by a grant from the NIH-DK129979, in part from the Peter O'Donnell Foundation, the Michael J. Fox Foundation, Jim Holland-Backcountry Foundation, and in part from a grant from the Army Research Office accomplished under Cooperative Agreement Number W911NF-19-2-0333.

\textbf{Parkinson data:} These data were generated by participants of The Michael J. Fox Foundation for Parkinson’s Research Mobile or Wearable Studies. They were obtained as part of the Biomarker \& Endpoint Assessment to Track Parkinson’s Disease DREAM Challenge (through Synapse ID syn20825169) made possible through partnership of The Michael J. Fox Foundation for Parkinson’s Research, Sage Bionetworks, and BRAIN Commons.
\end{ack}

%% file: tex/Appendix.tex
\section{MATHEMATICAL THEORY AND PROOF}
In this section, we provide the details deferred from the main paper, including the approximation of the $\mathcal{L}^{max}$ function (Section 2.3) and the derivation of the signal distribution $p((g_k)_T)$ (Section 2.5).

\subsection{Approximation of $\mathcal{L}^{max}$}
We aim to prove the approximation formula for $\mathcal{L}^{max}$:
\begin{equation}
\mathcal{L}^{max}(\mathcal{C}, G, S^m) \approx \log p_\mathcal{N}(Y|0, B\sigma^2 I + Q^{\mathcal{C}, G}) - \frac{1}{2\sigma^2B}\sum_{i=1}^B \Tr(K_{T_iT_i}-Q_{T_iT_i})
\end{equation}

We now state the approximation result as the following lemma:

\begin{lemma}\label{lemma:ELBO_max_formula}
Let $G = \set{g(t)}_{t \geq 0}$ be a stochastic process with the underlying signal $g$. Assume the data collection $\mathcal{C}$ for the process $G$ consists of $B$ noisy time series $\set{y^i}_{i=1}^B$ with data points $(y^i)_{T_i}$, where $(y^i)_{T_i} \sim \mathcal{N}(g_{T_i}, \sigma I_{|T_i|})$. This assumption, from Section 2.1, implies the data points have Gaussian noise with variance $\sigma$ around the signal $g$. We approximate $G$ by a Gaussian process with mean function $\mu:\mathbb{R} \to \mathbb{R}$ and kernel $K: \mathbb{R} \times \mathbb{R} \to \mathbb{R}$. We further assume the mean vectors $\mu$ can be approximately rescaled by $K$: $\mu_T \approx K_{TS}K_{SS}^{-1}\mu_S$.

Let $S^m$ be an $m$-element set of timestamps. Recall the kernel matrices $K_{T_iT_i}$, $K_{S^mT_i}$, and $K_{T_iS^m}$, defined as follows: $K_{T_iT_i} = (K(t, s))_{t \in T_i, s \in T_i}$, $K_{S^mT_i} = (K(t, s))_{t \in S^m, s \in T_i}$, and $K_{T_iS^m} = (K(t, s))_{t \in T_i, s \in S^m}$. Also, recall the $|T_i|$-by-$|T_i|$ matrix $Q_{T_iT_i}: = K_{T_iS^m}(K_{S^mS^m})^{-1}K_{S^mT_i}$ for $ i \in \overline{1, B}$. Additionally, recall the data vector $Y$ and the joint matrix $Q^{\mathcal{C}, G}$:
\begin{equation}
Y = \begin{bmatrix} y^1 \\ \vdots \\ y^B \end{bmatrix},\ Q^{\mathcal{C}, G} = \begin{bmatrix}
        Q_{T_1T_1} & 0 & 0\\
        0 & \ddots & 0\\
        0 & 0 & Q_{T_BT_B}
\end{bmatrix}
\end{equation}

Then $\mathcal{L}^{max}$ defined in Section 2.2 has the approximate form:
\begin{equation}\label{eqn:ELBO_max_formula_1}
\mathcal{L}^{max}(\mathcal{C}, G, S^m) \approx \log p_\mathcal{N}(Y|0, B\sigma^2 I + Q^{\mathcal{C}, G}) - \frac{1}{2\sigma^2B}\sum_{i=1}^B \Tr(K_{T_iT_i}-Q_{T_iT_i})
\end{equation}
where $p_\mathcal{N}(X|\mu, \Sigma)$ denotes the density function of a Gaussian random variable $X$ with mean $\mu$ and covariance matrix $\Sigma$. 

Furthermore, the optimal variational distribution $\phi^* = \argmax_\phi \mathcal{L}(\mathcal{C}, G, S^m, \phi)$ (see Section 2.2) is a Gaussian distribution of the form:
\begin{equation}\label{eqn:optimal_variational_distribution}
\phi^*(g_{S^m}) = \mathcal{N}\Bigg(\sigma^{-2}K_{S^mS^m}\Sigma\left(\frac{1}{B}\sum_{i=1}^B K_{S^mT_i}y^i\right), K_{S^mS^m}\Sigma K_{S^mS^m} \Bigg)
\end{equation}
where $\Sigma = \Lambda^{-1}$ with $\Lambda := K_{S^mS^m} + \dfrac{\sigma^{-2}}{B}\sum_{i=1}^B K_{S^mT_i}K_{T_iS^m}$.
\end{lemma}

\begin{proof}
Define the conditional mean signal vector $\alpha_i = \E[g_{T_i}|g_{S^m}]$. From the rescaled mean signals approximation, the conditional distribution of $g_T$ given $g_S$ is Gaussian with the following mean and variance:
\begin{equation}\label{gaussian_process_conditional_formula}
p(g_T|g_S, T, S) = \mathcal{N}(K_{TS}K_{SS}^{-1} g_S, K_{TT} - K_{TS} K_{SS}^{-1} K_{ST})
\end{equation}
As a result, $\alpha_i = K_{T_iS^m}(K_{S^mS^m})^{-1}g_{S^m}$. Then, following the derivation from \citep{pmlr-v5-titsias09a}, individual terms in Definition 1 can be approximated as follows:
\begin{align*}
&\int p(g_{T_i}|g_{S^m})\phi(g_{S^m}) \log \frac{p(y^i|g_{T_i})p(g_{S^m})}{\phi(g_{S^m})} dg_{T_i}dg_{S^m} \\
&= \int \phi(g_{S^m}) \bigg( \int p(g_{T_i}|g_{S^m}) \log p(y^i|g_{T_i})dg_{T_i} + \log \frac{p(g_{S^m})}{\phi(g_{S^m})} \bigg) dg_{S^m} \\
&\approx \int \phi(g_{S^m}) \bigg( \log p_\mathcal{N}(y^i|\alpha_i, \sigma I_{|T_i|}) - \frac{1}{2\sigma^2} \Tr(K_{T_iT_i} - Q_{T_iT_i}) + \log \frac{p(g_{S^m})}{\phi(g_{S^m})} \bigg)dg_{S^m} \\
&= \int \phi(g_{S^m}) \log \frac{p_\mathcal{N}(y|\alpha_i, \sigma I_{|T_i|})p(g_{S^m})}{\phi(g_{S^m})} dg_{S^m} - \frac{1}{2\sigma^2} \Tr(K_{T_iT_i} - Q_{T_iT_i})
\end{align*}

Let $A$ be the combined mean signal vector $A:=\begin{bmatrix} \alpha_1 \\ \vdots \\ \alpha_B \end{bmatrix}$. Using the above approximation for individual terms, we upper-bound the function $\mathcal{L}(S, G, T^m, \phi)$ (see Definition 1 in Section 2.2) as follows:
\begin{align*}
&\mathcal{L}(S, G, T^m, \phi) \\
&\approx \sum_{i=1}^B \frac{1}{B} \int \phi(g_{S^m}) \log \frac{p_\mathcal{N}(y^i|g_{T_i}, \sigma^2I)p(g_{S^m})}{\phi(g_{S^m})} dg_{S^m} -\frac{1}{2\sigma^2B}\sum_{i=1}^B \Tr(K_{T_iT_i}-Q_{T_iT_i}) \\
&= \int \phi(g_{S^m}) \log\Bigg(\Bigg(\prod_i p_\mathcal{N}(y^i|\alpha_i, \sigma^2I)\Bigg)^{1/B} \frac{p(g_{S^m})}{\phi(g_{S^m})}\Bigg) dg_{S^m} -\frac{1}{2\sigma^2B}\sum_{i=1}^B \Tr(K_{T_iT_i}-Q_{T_iT_i})\\
&\leq \log \int \Bigg(\prod_{i=1}^B p_\mathcal{N}(y^i|\alpha_i, \sigma^2I)\Bigg)^{1/B} p(g_{S^m}) dg_{S^m} - \frac{1}{2\sigma^2B}\sum_{i=1}^B \Tr(K_{T_iT_i}-Q_{T_iT_i}) \\
&= \log \int p_\mathcal{N}(Y|A, B\sigma^2)p(g_{S^m}) dg_{S^m} - \frac{1}{2\sigma^2B}\sum_{i=1}^B \Tr(K_{T_iT_i}-Q_{T_iT_i})\\
&= \log p_\mathcal{N}(Y|0, B\sigma^2 I + Q^{\mathcal{C}, G}) - \frac{1}{2\sigma^2B}\sum_{i=1}^B \Tr(K_{T_iT_i}-Q_{T_iT_i})
\end{align*}
The only inequality for this bound is due to Jensen inequality. This upper-bound no longer depends on the variational distribution $\phi$ and only depends on the timestamps in $S^m$. As a result, by definition of $\mathcal{L}^{max}$, we obtain the \cref{eqn:ELBO_max_formula_1}. Moreover, for this bound, the equality holds when:
\begin{align*}
\phi^*(g_{S^m}) &\propto \prod_{i=1}^B p_\mathcal{N}(y^i|\alpha_i, \sigma^2I)^{1/B}p(g_{S^m}) \\
&\propto \exp\Bigg( \frac{\sigma^{-2}}{B} \sum_{i=1}^B  \Bigg( (y^i)^T K_{T_iS^m} (K_{S^mS^m})^{-1}g_{S^m} \Bigg) - \frac{1}{2} (g_{S^m})^T \times \\
&\bigg(\frac{\sigma^{-2}}{B}\sum_{i=1}^B \bigg( (K_{S^mS^m})^{-1}K_{S^mT_i}K_{T_iS^m}(K_{S^mS^m})^{-1} \bigg) + (K_{S^mS^m})^{-1} \bigg) \times g_{S^m} \Bigg) \nonumber 
\end{align*}
Hence, $\phi^*$ is (approximately) a Gaussian distribution with the following mean and variance:
\begin{equation}
\phi^*(g_{S^m}) = \mathcal{N}\Bigg(\sigma^{-2}K_{S^mS^m}\Lambda\left(\frac{1}{B}\sum_{i=1}^B K_{S^mT_i}y^i\right), K_{S^mS^m}\Lambda K_{S^mS^m} \Bigg)
\end{equation}
\end{proof}

\subsection{Calculation of the Distribution $p((g_k)_T)$}
Recall that $S^{m, k}$ represents the most informative timestamps for the underlying stochastic process $G_k$, associated with the time series data collection $\mathcal{C}_k$ (see Section 2.1). Furthermore, $G_k$ is approximated by a Gaussian process with a parameterized kernel function $K^{\eta_k}$ (see section 2.4). 

We now provide a detailed calculation/approximation of the distribution of the underlying signal $p((g_k)_T)$ and its mean $p_k = \E[(g_k)_T]$, referred to as the \textbf{predicted mean signal} (see Section 2.5). For $k \in \overline{1, L}$, the predicted distribution of the signal vector $(g_k)_T$ is obtained by marginalizing over the signal $(g_k)_{S^{m, k}}$ at the most informative timestamps $S^{m, k}$ for process $G_k$:
\begin{equation}\label{eqn:predicted_dist_marginalization}
p((g_k)_T) = \int p((g_k)_T|(g_k)_{S^{m, k}}) \phi^*((g_k)_{S^{m, k}}) d(g_k)_{S^{m, k}}
\end{equation}
Here the optimal variational distribution $\phi^*$ has the approximate form defined in \cref{eqn:optimal_variational_distribution}. Specifically, $\phi^*$ can be approximated by a Gaussian distribution with the following mean $\mu_k$ and covariance matrix $A_k$, based on \cref{lemma:ELBO_max_formula}:
\begin{align}
\mu_k &= \sigma^{-2}K^{\eta_k}_{S^{m, k}S^{m, k}}\Sigma\left(\frac{1}{B_k}\sum_{i=1}^{B_k} K^{\eta_k}_{S^{m, k}T_{i, k}}y^i\right) \\
A_k &= K^{\eta_k}_{S^{m, k}S^{m, k}}\Sigma K^{\eta_k}_{S^{m, k}S^{m, k}}
\end{align}
where $\Sigma = \Lambda^{-1}$ with $\Lambda := K^{\eta_k}_{S^{m, k}S^{m, k}} + \dfrac{\sigma^{-2}}{B_k}\sum_{i=1}^{B_k} K^{\eta_k}_{S^{m, k}T_{i, k}}K^{\eta_k}_{T_{i,k}S^{m, k}}$.

With this approximation for $\phi^*$, \cref{eqn:predicted_dist_marginalization} simplifies to an integral of two Gaussian distributions. An explicit calculation shows that the distribution of $(g_k)_T$ is approximated by a Gaussian distribution with the following mean and variance: 
\begin{align} \label{predicted_dist}
p_k = \E[(g_k)_T] &= K^{\eta_k}_{TS^{m, k}}(K^{\eta_k}_{S^{m, k}S^{m, k}})^{-1}\mu_k \\
\Var[(g_k)_T] &= K^{\eta_k}_{TT}-K^{\eta_k}_{TS^{m, k}}(K^{\eta_k}_{S^{m, k}S^{m, k}})^{-1}K^{\eta_k}_{S^{m, k}T} \nonumber \\
&+ K^{\eta_k}_{TS^{m, k}} (K^{\eta_k}_{S^{m, k}S^{m, k}})^{-1} A_k (K^{\eta_k}_{S^{m, k}S^{m, k}})^{-1} K^{\eta_k}_{S^{m, k}T}
\end{align}

\section{DATASETS}
\subsection{Basic Datasets}
Twelve publicly available time-series datasets were sourced from the UCR archive \citep{Bagnall2017-lb}. These datasets are the following, with their respective IDs from 1 to 12: Chinatown, ECGFiveDays, FreezerSmallTrain, GunPointOldVersusYoung, HouseTwenty, InsectEPGRegularTrain, ItalyPowerDemand, Lightning7, MoteStrain, PowerCons, SonyAIBORobotSurface2, and UWaveGestureLibraryAll.

\subsection{Pronunciation Audio Dataset}
For the \textbf{Pronunciation Audio} dataset, the audio samples were obtained from publicly available pronunciation recordings \citep{forvo}. The original pronunciation audio files are included in the folder \textbf{data/audio}, and all processing steps related to these files can be found in the provided Python file \textbf{data\_processing.py}. This file contains the full preprocessing pipeline for converting raw audio into time-series data, which was used for benchmarking and experimentation.

\subsection{Parkinson's Disease Sensor Dataset}
During data processing, all personal identifying information (PII) has been thoroughly removed from the dataset to ensure privacy and data security. The sensor data has been aggregated to a per-second level, meaning the original, unaggregated data cannot be recovered, thereby minimizing any risk of data exposure. The processing steps for this dataset are available in the Python file \textbf{parkinson\_data\_processing.py}. This code generates the second-level processed data, which serves as input for all benchmarking algorithms, including Motion Code.

To access the full original data and labeled datasets, researchers must apply for a separate license. To apply for access to the original datasets, follow the instructions provided at: \href{https://www.synapse.org/Synapse:syn20825169/wiki/600903}{https://www.synapse.org/Synapse:syn20825169/wiki/600903}.

In addition, we provide general information for both the Pronunciation Audio data and the Parkinson’s sensor data as processed by us in \cref{dataset_description} below:

\begin{table}[ht]
\begin{center}
\caption{\label{dataset_description}Descriptions of 3 Datasets Processed by Authors.}
\vspace{2mm}
\begin{tabular}{p{3cm}p{1cm}p{1cm}p{1.5cm}p{5cm}} 
\toprule
Dataset & Train & Test & Length & Description \\ [6pt]
\midrule
Pronunciation Audio & 18 & 18 & 80-100 & Amplitude values of the audio datasets for the pronunciations of 2 words with different accents. \\
PD setting 1 & 20 & 322 & 257-1665 & Parkinson's disease sensor data focusing on understanding recovery stage\\
PD setting 2 & 24 & 429 & 208-1665 & Parkinson's disease sensor data focusing on detecting tremor pattern\\
\bottomrule
\end{tabular}
\end{center}
\end{table}

\section{ADDITIONAL FIGURES}
\subsection{Interpretable Features}
\begin{figure}[htb]
  \centering
  \subfloat[October to March Period]{\includegraphics[scale=0.4]{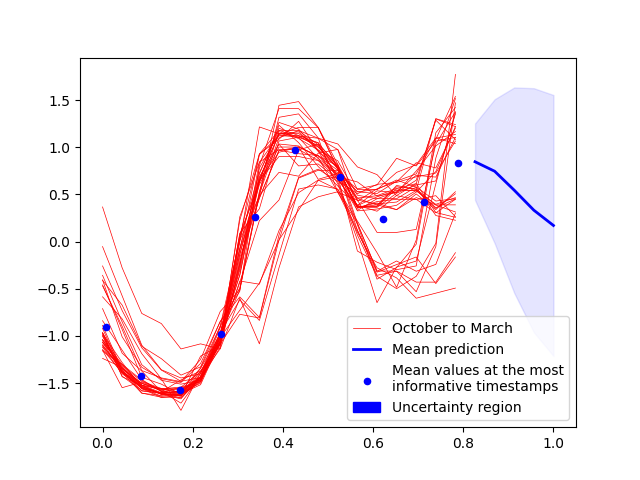}}
  \hfil
  \subfloat[April to September Period]{\includegraphics[scale=0.4]{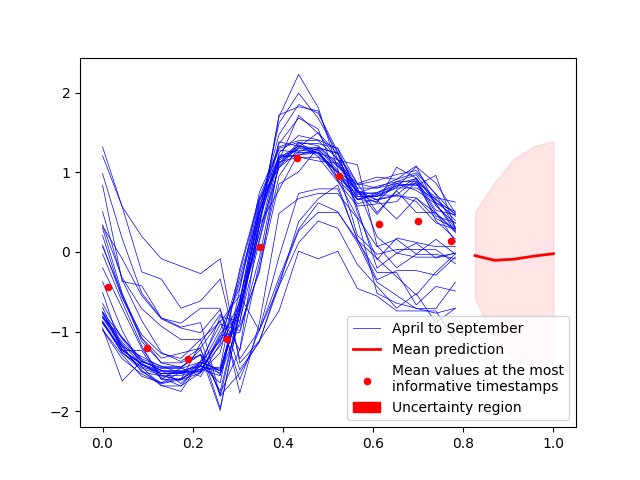}}
  \caption{Forecasting with Uncertainty and Interpretable Features for \textbf{ItalyPowerDemand}.} \label{fig:most_informative_timestamp_prediction_appendix}
\end{figure}

\begin{figure}[htb]
  \centering
  \subfloat[Cement]{\includegraphics[scale=0.4]{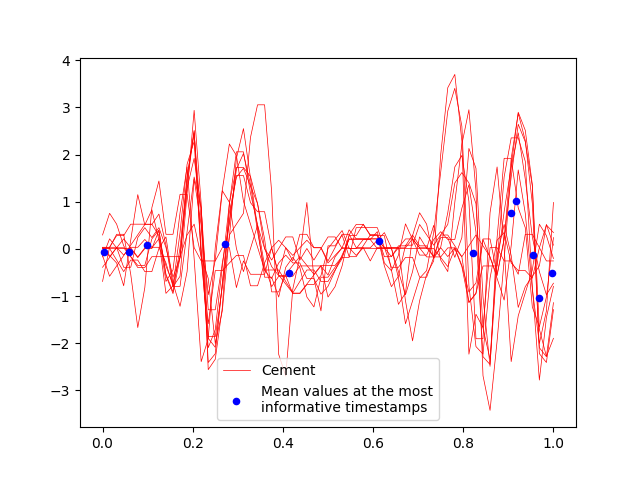}}
  \hfil
  \subfloat[Carpet]{\includegraphics[scale=0.4]{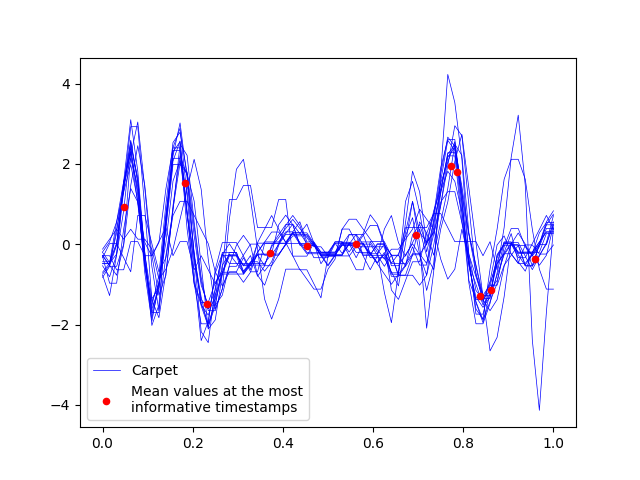}}
  \caption{Interpretable Features for \textbf{SonyAIBORobotSurface2}.} \label{fig:most_informative_timestamp_prediction1_appendix}
\end{figure}

\begin{figure}[htb]
  \centering
  \subfloat[Class 1]{\includegraphics[scale=0.28]{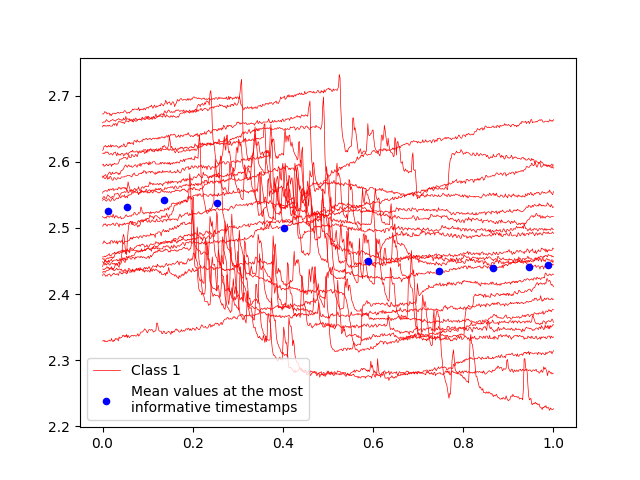}}
  \hfil
  \subfloat[Class 2]{\includegraphics[scale=0.28]{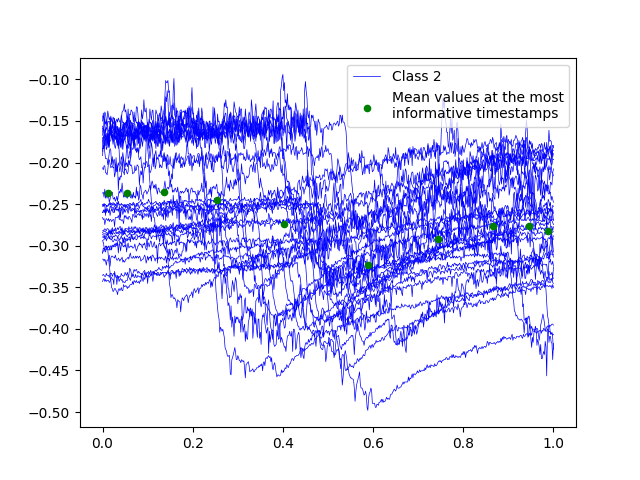}}
  \hfil
  \subfloat[Class 3]{\includegraphics[scale=0.28]{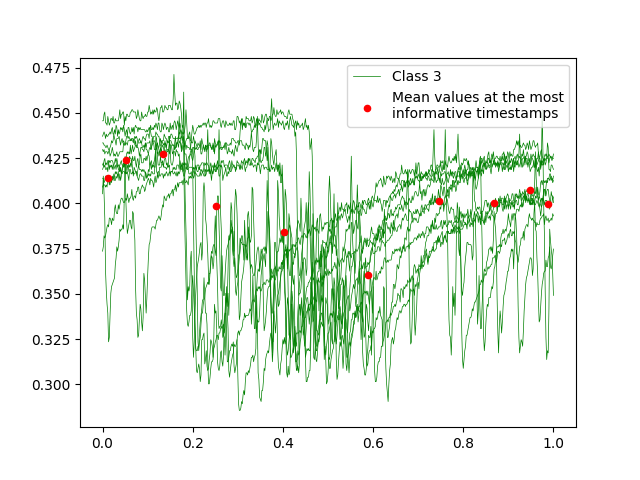}}
  \caption{Interpretable Features for \textbf{InsectEPGRegularTrain}.} \label{fig:most_informative_timestamp_prediction2_appendix}
\end{figure}

\newpage

\subsection{Forecasting with Uncertainty}
\begin{figure}[htb]
  \centering
  \subfloat[October to March Period]{\includegraphics[scale=0.4]{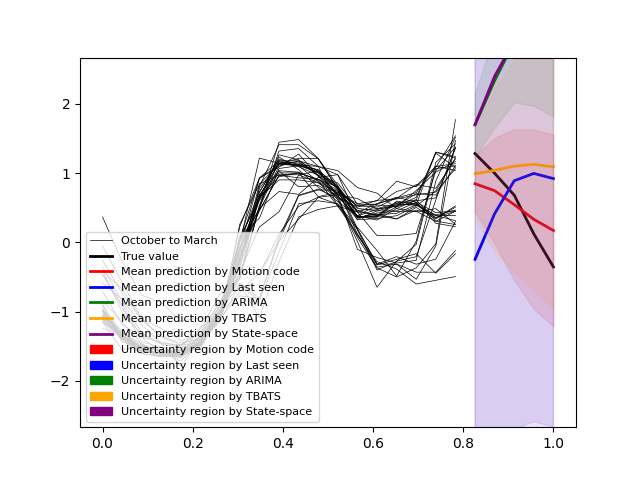}}
  \hfil
  \subfloat[April to September Period]{\includegraphics[scale=0.4]{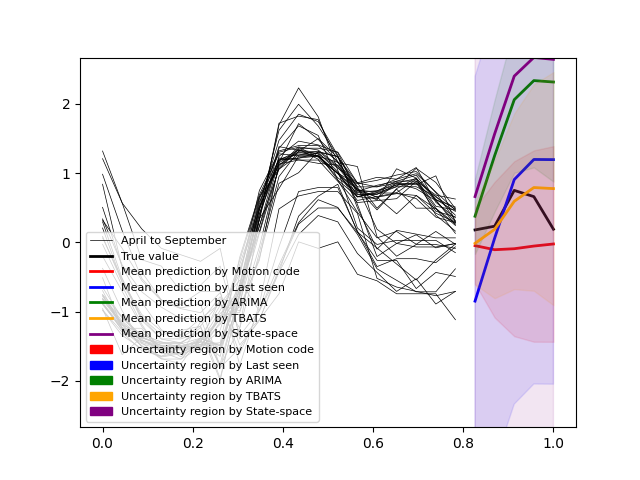}}
  \caption{Forecasting with Uncertainty on \textbf{ItalyPowerDemand} from 5 Forecasting Algorithms}\label{fig:var_ItalyPowerDemand}
\end{figure}

\begin{figure}[htb]
  \centering
  \subfloat[Warm Season in PowerCons]{\includegraphics[scale=0.4]{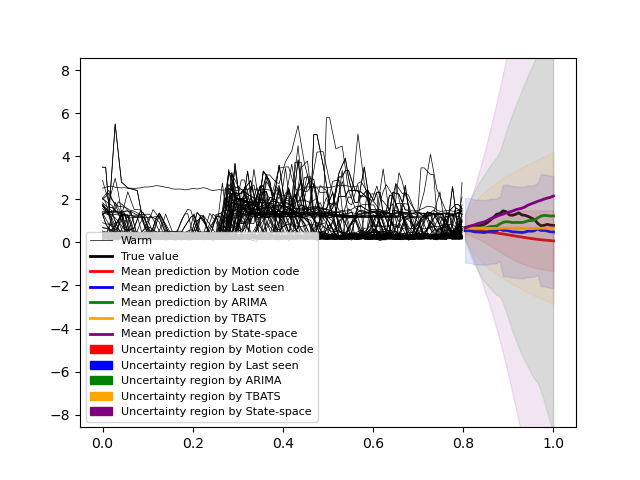}}
  \hfil
  \subfloat[Cold Season in PowerCons]{\includegraphics[scale=0.4]{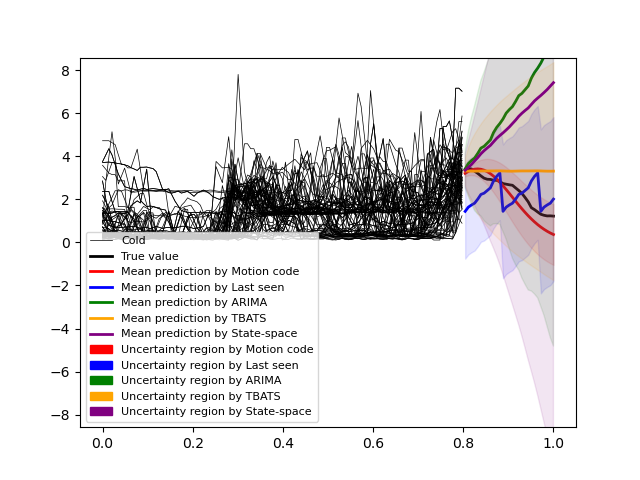}}
  \caption{Forecasting with Uncertainty on \textbf{PowerCons} from 5 Forecasting Algorithms}\label{fig:var_PowerCons}
\end{figure}